\documentclass[fleqn,10pt]{wlpeerj}

\pdfoutput=1

\usepackage{graphicx}
\usepackage{hyperref}
\usepackage{nameref}
\usepackage{enumitem}

\makeatletter
\def\namedlabel#1#2{\begingroup
    #2%
    \def\@currentlabel{#2}%
    \phantomsection\label{#1}\endgroup
}
\makeatother

\usepackage{amsmath}
\usepackage{amssymb}
\usepackage{amsthm}

\makeatletter
\newtheorem*{rep@theorem}{\rep@title}
\newcommand{\newreptheorem}[2]{%
\newenvironment{rep#1}[1]{%
 \def\rep@title{#2 \ref{##1}}%
 \begin{rep@theorem}}%
 {\end{rep@theorem}}}
\makeatother

\usepackage{subcaption}

\usepackage{float}
\usepackage{algorithm}
\usepackage{algpseudocode}

\newtheorem{theorem}{Theorem}
\newreptheorem{theorem}{Theorem}
\newtheorem{definition}{Definition}

\DeclareMathOperator*{\argmin}{\mathrm{arg\,min}}

\DeclareMathOperator*{\Argmin}{\mathrm{Arg\,min}}

\DeclareMathOperator*{\E}{\mathbb{E}}

\title{Adaptive Divergence for Rapid Adversarial Optimization}
\author[1]{Maxim~Borisyak}
\author[1]{Tatiana~Gaintseva}
\author[1, 2, 3]{Andrey~Ustyuzhanin}%
\affil[1]{Laboratory of Methods for Big Data Analysis, National Research University Higher School of Economics, 20 Myasnitskaya st, Moscow 101000, Russia}
\affil[2]{Physics department, Imperial College, South Kensington, London SW7 2AZ, United Kingdom}
\affil[3]{Department of multidisciplinary research, National University of Science and Technology MISIS, 4 Leninsky av, Moscow 119049, Russia}
\corrauthor[1]{Maxim Borisyak}{mborisyak@hse.ru}

\begin{abstract}
    Adversarial Optimization (AO) provides a reliable, practical way to match two implicitly defined distributions, one of which is usually represented by a sample of real data, and the other is defined by a generator. Typically, AO involves training of a high-capacity model on each step of the optimization. In this work, we consider computationally heavy generators, for which training of high-capacity models is associated with substantial computational costs. To address this problem, we introduce a novel family of divergences, which varies the capacity of the underlying model, and allows for a significant acceleration with respect to the number of samples drawn from the generator.

    We demonstrate the performance of the proposed divergences on several tasks, including tuning parameters of a physics simulator, namely, Pythia event generator.
\end{abstract}
\begin{document}

\flushbottom
\maketitle
\thispagestyle{empty}

\section{Introduction}

Adversarial Optimization (AO), introduced in Generative Adversarial Networks~\citep{goodfellow2014generative}, became popular in many areas of machine learning and beyond with applications ranging from generative~\citep{radford2015unsupervised} and inference tasks~\citep{dumoulin2016adversarially}, improving image quality~\citep{isola2017image} to tuning stochastic computer simulations~\citep{avo}.

AO provides a reliable, practical way to match two implicitly defined distributions, one of which is typically represented by a sample of real data, and the other is represented by a parameterized generator. Matching of the distributions is achieved by minimizing a divergence between these distribution, and estimation of the divergence involves a secondary optimization task, which, typically, requires training a model to discriminate between these distributions. The model is referred to as discriminator or critic (for simplicity, we use term discriminator everywhere below).

Training a high-capacity model, however, is computationally expensive~\citep{metz2016unrolled} as each step of divergence minimization is accompanied by fitting the discriminator; therefore, adversarial training often requires significantly more computational resources than, for example, a classification model with a comparable architecture of the networks \footnote{For instance, compare training times, network capacities and computational resources reported by \cite{simonyan2014very} and \cite{choi2018stargan}.}. Nevertheless, in conventional settings like GAN, this problem is not pronounced for at least two reasons. Firstly, the generator is usually represented by a deep neural network, and sampling is computationally cheap; thus, for properly training the discriminator, a training sample of sufficient size can be quickly drawn. Secondly, GAN training procedures are often regarded not as minimization of a divergence, but as game-like dynamics~\citep{li2017towards, mescheder2018training}; such dynamics typically employ gradient optimization with small incremental steps, which involve relatively small sample sizes for adapting previous discriminator to an updated generator configuration.

Computational costs of AO becomes significant when sampling from the generator is computationally expensive, or optimization procedure does not operate by performing small incremental steps~\citep{metz2016unrolled}. One of the practical examples of such settings is fine-tuning parameters of complex computer simulations. Such simulators are usually based on physics laws expressed in computational mathematical forms like differential or stochastic equations. Those equations relate input or initial conditions to the observable quantities under conditions of parameters that define physics laws, geometry or other valuable property of the simulation; these parameters do not depend on inputs or initial conditions. It is not uncommon that such simulations have very high computational complexity. For example, the simulation of a single proton collision event in the CERN ATLAS detector takes several minutes on a single core CPU~\citep{bouhova2010atlas}. Due to typically high dimensionality, it takes a considerable amount of samples for fine-tuning, which in turn increases the computational burden.

Another essential property of such computer simulations is the lack of gradient information over the simulation parameters. Computations are represented by sophisticated computer programs, which are challenging to differentiate\footnote{There are ways to estimate gradients of such programs, for example, see~\citep{baydin2018efficient}. However, all methods known to the authors require training a surrogate, which encounters the problem of the expensive sampling procedures mentioned above.}. Thus, global black-box optimization methods are often employed; Bayesian Optimization is one of the most popular approaches.

In this work, we introduce a novel family of divergences which enables faster optimization convergence measured by the number of samples drawn from the generator.  The variation of the underlying discriminator model capacity during optimization leads to a significant speed-up. The proposed divergence family suggests using low-capacity models to compare distant distributions (typically, at early optimization steps), and the capacity gradually grows as the distributions become closer to each other. Thus it allows for a significant acceleration of the initial stages of optimization. Additionally, the proposed family of divergences is broad, which offers a wide range of opportunities for further research.

We demonstrate the basic idea with some toy examples, and with a real challenge of tuning Pythia event generator \citep{pythia64, pythia82} following \cite{avo}~and~\cite{ilten2017event}. We consider physics-related simulations; nevertheless, all proposed methods are simulation-agnostic.

\section{Background}
Adversarial Optimization, initially introduced for Generative Adversarial Networks (GAN)~\citep{goodfellow2014generative}, offers a general strategy for matching two distributions. Consider feature space $\mathcal{X}$, ground-truth distribution $P$ and parametrized family of distributions $Q_\psi$ implicitly defined by a generator with parameters $\psi$. Formally, we wish to find such $\psi^*$, that $P = Q_{\psi^*}$ almost everywhere. AO achieves that by minimizing a divergence or a distance between $P$ and $Q_\psi$ with respect to $\psi$. One of the most popular divergences is Jensen-Shannon divergence:
\begin{multline}
    \mathrm{JSD}(P, Q_\psi) = \frac{1}{2}\left[ \mathrm{KL}(P \| M_\psi) + \mathrm{KL}(Q_\psi \| M_\psi) \right] =\\ \log 2 - \min_{f \in \mathcal{F}} \left[ -\frac{1}{2}\E_{x \sim P} \log f(x) - \frac{1}{2}\E_{x \sim Q_\psi} \log(1 - f(x)) \right] =\\ \log 2 - \min_{f \in \mathcal{F}} L(f, P, Q_\psi) \label{eq:jsd};
\end{multline}
where: $\mathrm{KL}$ --- Kullback-Leibler divergence, $M_\psi(x) = \frac{1}{2}\left( P(x) + Q_\psi(x)\right)$, $L$ --- cross-entropy loss function, and $\mathcal{F} = \{ f : \mathcal{X} \to [0, 1] \}$ is the set of all possible discriminators. Expression~\eqref{eq:jsd} provides a practical way to estimate Jensen-Shannon divergence by training a powerful discriminator to distinguish samples from $P$ against samples from $Q_\psi$.

In classical GAN optimization, both generator and discriminator are represented by differentiable neural networks. Hence, a subgradient of $\mathrm{JSD}(P, Q_\psi)$ can be easily computed~\citep{goodfellow2014generative}. The minimization of the divergence can be performed by a gradient method, and the optimization procedure goes iteratively following those steps:
\begin{itemize}
    \item using parameters of the discriminator from the previous iteration as an initial guess, adjust $f$ by performing several steps of the gradient descent to minimize $\mathcal{L}(f, P, Q_\psi)$;
    \item considering $f$ as a constant, compute the gradient of $\mathcal{L}(f, P, Q_\psi)$ by $\psi$, perform one step of the gradient ascent.
\end{itemize}


For computationally heavy generators, gradients are usually practically unfeasible; therefore, we consider black-box optimization methods. One of the most promising methods for black-box AO is Adversarial Variational Optimization~\citep{avo}, which combines AO with Variational Optimization~\citep{wierstra2014natural}. This method improves upon conventional Variational Optimization (VO) over Jensen-Shannon divergence by training a single discriminator to distinguish samples from ground-truth distribution and samples from a mixture of generators, where the mixture is defined by the search distribution of VO. This eliminates the need to train a classifier for each individual set of parameters drawn from the search distribution.

Bayesian Optimization (BO)~\citep{mockus2012bayesian} is another commonly used black-box optimization method, with applications including tuning of complex simulations~\citep{ilten2017event}. As we demonstrate in section 5, BO can be successfully applied for Adversarial Optimization.

\section{Adaptive Divergence}

Notice, that in Equation~\eqref{eq:jsd}, minimization is carried over the set of all possible discriminators $\mathcal{F} = \{f : \mathcal{X} \mapsto [0, 1] \}$. In practice, this is intractable and set $\mathcal{F}$ is approximated by a model such as Deep Neural Networks.
Everywhere below, we use terms 'low-capacity' and 'high-capacity' to describe the set of feasible discriminator functions: low-capacity models are either represent a narrow set of functions (e.g., logistic regression, shallow decision trees) or are heavily regularized (see Section~\ref{sec:impl} for more examples of capacity regulation); high-capacity models are sufficient for estimating $\mathrm{JSD}$ for an Adversarial Optimization problem under consideration.

In conventional GAN settings, the generator is represented by a neural network, sampling is computationally cheap, and usage of high-capacity discriminators is satisfactory. In our case, as was discussed above, simulations tend to be computationally heavy, which, combined with a typically slow convergence of black-box optimization algorithms, might make AO with a high-capacity model practically intractable.

The choice of the model has its trade-off: high-capacity models provide good estimations of $\mathrm{JSD}$, but, generally, require large sample sizes to be properly trained. In contrast, low-capacity models tend to require fewer samples for training; however, they might provide biased estimations. For example, if the classifier is represented by a narrow set of functions $M \subseteq \mathcal{F}$, then quantity:
\begin{equation}
    D_M(P, Q) = \log 2 - \min_{f \in M} L(f, P, Q);\label{eq:model-pd}
\end{equation}
might no longer be a divergence, so we refer to it as \textit{pseudo-divergence}.

\begin{definition}
    \label{def:pseudo-divergence}
    A function $D : \Pi(\mathcal{X}) \times \Pi(\mathcal{X}) \to \mathbb{R}$ is a pseudo-divergence, if:
    \begin{description}
        \item [\namedlabel{prop:pd-non-negativity}{(\textbf{P1})}] $\forall P, Q \in \Pi(\mathcal{X}): D(P, Q) \geq 0$;
        \item [\namedlabel{prop:pd-equality}{(\textbf{P2})}] $\forall P, Q \in \Pi(\mathcal{X}): (P = Q) \Rightarrow D(P, Q) = 0$;
    \end{description}
    where $\Pi(\mathcal{X})$ --- set of all probability distributions on space $\mathcal{X}$.
\end{definition}

It is tempting to use a pseudo-divergence $D_M$ produced by a low-capacity model $M$ for Adversarial Optimization, however, a pseudo-divergence might not guarantee proper convergence as there might exist such $\psi \in \Psi$, that $\mathrm{JSD}(P, Q_\psi) > 0$, while $\mathrm{D}(P, Q_\psi) = 0$. For example, naive Bayes classifier is unable to distinguish between $P$ and $Q$ that have the same marginal distributions. Nevertheless, if model $M$ is capable of distinguishing between $P$ and some $Q_\psi$, $D_M$ still provides information about the position of the optimal parameters in the configuration space $\psi^*$ by narrowing search volume, \citet{ilten2017event}~offers a good demonstration of this statement.


The core idea of this work is to replace Jensen-Shannon divergence with a so-called adaptive divergence that gradually adjusts model capacity depending on the 'difficulty' of the classification problem with the most 'difficult' problem being distinguishing between two equal distributions. Formally, this gradual increase of model complexity can be captured by the following definitions.

\begin{definition}
    \label{def:family}
    A family of pseudo-divergences $\mathcal{D} = \{ D_\alpha : \Pi(\mathcal{X}) \times \Pi(\mathcal{X}) \to \mathbb{R} \mid \alpha \in [0, 1] \}$ is ordered and complete with respect to Jensen-Shannon divergence if:
    \begin{description}
        \item [\namedlabel{prop:family-pd}{(D0)}] $D_\alpha$ is a pseudo-divergence for all $\alpha \in [0, 1]$;
        \item [\namedlabel{prop:family-ordered}{(D1)}] $\forall P, Q \in \Pi(\mathcal{X}): \forall 0 \leq \alpha_1 < \alpha_2 \leq 1: D_{\alpha_1}(P, Q) \leq D_{\alpha_2}(P, Q)$;
        \item [\namedlabel{prop:family-complete}{(D2)}] $\forall P, Q \in \Pi(\mathcal{X}): D_1(P, Q) = \mathrm{JSD}(P, Q)$.
    \end{description}
\end{definition}

There are numerous ways to construct a complete and ordered w.r.t. JSD family of pseudo-divergences. In the context of Adversarial Optimization, we consider the following three methods. The simplest one is to define a nested family of models $\mathcal{M} = \{ M_\alpha \subseteq \mathcal{F} \mid \alpha \in [0, 1] \}$, (e.g., by changing number of hidden units of a neural network), then use pseudo-divergence~\eqref{eq:model-pd} to form a desired family.

Alternatively, for a parameterized model $M = \{ f(\theta, \cdot) \mid \theta \in \Theta \}$, one can use a regularization $R(\theta)$ to control 'capacity' of the model:
\begin{eqnarray}
    D_\alpha(P, Q) &=& \log 2 - L(f(\theta^*, \cdot), P, Q); \label{eq:reg-pd}\\
    \theta^* &=& \argmin_{\theta \in \Theta}L(f(\theta, \cdot), P, Q) + c(1 - \alpha) \cdot R(\theta); \nonumber
\end{eqnarray}
where $c: [0, 1] \to [0, +\infty)$ is a strictly increasing function and $c(0) = 0$.

The third, boosting-based method is applicable for a discrete approximation:
\begin{eqnarray}
    D_{c(i)}(P, Q) &=& \log 2 - L(F_i, P, Q) \label{eq:pd-boost};\\
    F_i &=& F_{i - 1} + \rho \cdot \argmin_{f \in B} L(F_{i - 1} + f, P, Q) \nonumber;\\
    F_0 &\equiv& \frac{1}{2} \nonumber;
\end{eqnarray}
where: $\rho$ --- learning rate, $B$ --- base estimator, $c: \mathbb{Z}_+ \to [0, 1]$ --- a strictly increasing function for mapping ensemble size onto $\alpha \in [0, 1]$.

Although Definition~\ref{def:family} is quite general, in this paper, we focus on families of pseudo-divergence produced in the manner similar to the examples above. All these examples introduce a classification algorithm parameterized by $\alpha$, then define pseudo-divergences $D_\alpha$ by substituting the optimal discriminator in Equation~\eqref{eq:jsd} with the discriminator trained in accordance with this classification algorithm with the parameter $\alpha$.
Of course, one has to make sure that the resulting family of pseudo-divergences is ordered and complete w.r.t. Jensen-Shannon divergence.
Appendix~\ref{sec:app-def} provides formal definitions and proofs for the examples above.

With this class of pseudo-divergences in mind, we refer to $\alpha$ as capacity of the pseudo-divergence $D_\alpha \in \mathcal{D}$ relative to the family $\mathcal{D}$, or simply as capacity if the family $\mathcal{D}$ is clear from the context. In the examples above, capacity of pseudo-divergence is directly linked to the capacity of underlying discriminator models: to the size of the model in equation~\eqref{eq:model-pd}, to the strength of the regularization in equation~\eqref{eq:reg-pd} (which, similar to the previous case, effectively restricts the size of the set of feasible models) or to the size of the ensemble for a boosting-based family of divergences in equation~\eqref{eq:pd-boost}.

Finally, we introduce a function that combines a family of pseudo-divergences into a single divergence.

\begin{definition}
    \label{def:ad}
    If a family of pseudo-divergences $\mathcal{D} = \{ D_\alpha \mid \alpha \in [0, 1] \}$ is ordered and complete  with respect to Jensen-Shannon divergence, then adaptive divergence $\mathrm{AD}_\mathcal{D}$ produced by $\mathcal{D}$ is defined as:
    \begin{equation}
        \mathrm{AD}_\mathcal{D}(P, Q) = \inf \left\{ D_\alpha(P, Q) \mid \mathrm{D}_\alpha(P, Q) \geq (1 - \alpha) \log 2\right\}. \label{eq:ad}
    \end{equation}
\end{definition}
We omit index in $\mathrm{AD}_\mathcal{D}$ when the family $\mathcal{D}$ is clear from the context or is not important.

Note, that due to property~(\textbf{D1}), $\mathrm{D}_\alpha(P, Q)$ is a non-decreasing function of $\alpha$, while $(1 - \alpha) \log 2$ is a strictly decreasing one. Hence, if family $\mathcal{D}$ is such that for any two distributions $P$ and $Q$ $\mathrm{D}_\alpha(P, Q)$ is continuous w.r.t. $\alpha$, equation~\eqref{eq:ad} can be simplified: 
\begin{equation}
    \mathrm{AD}_\mathcal{D}(P, Q) = \mathrm{D}_{\alpha^*}(P, Q),
\end{equation}
where $\alpha^*$ is the root of the following equation:
\begin{equation}
    D_{\alpha}(P, Q) = (1 - \alpha)\log 2.
\end{equation}
A general procedure for computing $\mathrm{AD}_\mathcal{D}$ for this case is outlined in algorithm~\ref{algo:ad}.

Intuitively, an adaptive divergence $\mathrm{AD}_\mathcal{D}$ switches between members of $\mathcal{D}$ depending on the `difficulty' of separating $P$ and $Q$. For example, consider family $\mathcal{D}$ produced by equation~\eqref{eq:reg-pd} with a high-capacity neural network as model $M$ and $l_2$ regularization $R$ on its weights. For a pair of distant $P$ and $Q$, even a highly regularized network is capable of achieving low cross-entropy loss and, therefore, $\mathrm{AD}_\mathcal{D}$ takes values of the pseudo-divergence based on such network. As distribution $Q$ moves close to $P$, $\mathrm{AD}_\mathcal{D}$ lowers the regularization coefficient, effectively increasing the capacity of the underlying model.

\begin{algorithm}[tbp]
    \caption{General procedure for computing an adaptive divergence by grid search}
    \label{algo:ad}
    \begin{algorithmic}
        \Require
        $\mathcal{D} = \{ D_\alpha \mid \alpha \in [0, 1] \}$ --- ordered and complete w.r.t. Jensen-Shannon divergence family of pseudo-divergences; $\varepsilon$ --- tolerance; $P$, $Q$ --- input distributions
        \State $\alpha \leftarrow 0$;
        
        \While{$D_\alpha(P, Q) < (1 - \alpha) \log 2$}
            \State $\alpha \leftarrow \alpha + \varepsilon$
        \EndWhile
        \State \Return{$D_\alpha(P, Q)$}
    \end{algorithmic}
\end{algorithm}

The idea behind adaptive divergences can be viewed from a different angle. Given two distributions $P$ and $Q$, it scans producing family of pseudo-divergences, starting from $\alpha = 0$ (the least powerful pseudo-divergence), and if some pseudo-divergence reports high enough value, it serves as a `proof' of differences between $P$ and $Q$. If all pseudo-divergences from the family $\mathcal{D}$ report $0$, then $P$ and $Q$ are equal almost everywhere as the family always includes $\mathrm{JSD}$ as a member. Formally, this intuition can be expressed with the following theorem.

\begin{reptheorem}{th:divergence}
    If $\mathrm{AD}_\mathcal{D}$ is an adaptive divergence produced by a ordered and complete  with respect to Jensen-Shannon divergence family of pseudo-divergences $\mathcal{D}$, then for any two distributions $P$ and $Q$: $\mathrm{JSD}(P, Q) = 0$ if and only if  $\mathrm{AD}(P, Q) = 0$.
\end{reptheorem}

A formal proof of Theorem~\ref{th:divergence} can be found in Appendix~\ref{sec:proof}. Combined with the observation that $\mathrm{AD}(P, Q) \geq 0$ regardless of $P$ and $Q$, the theorem states that $\mathrm{AD}$ is a divergence in the same sense as $\mathrm{JSD}$. This, in turn, allows the use of adaptive divergences as a replacement for Jensen-Shannon divergence in Adversarial Optimization.

As can be seen from the definition, adaptive divergences are designed to utilize low-capacity pseudo-divergences (with underlying low-capacity models) whenever it is possible: for a pair of distant $P$ and $Q$ one needs to train only a low-capacity model to estimate $\mathrm{AD}$, using the most powerful model only to prove equality of distributions. As low-capacity models generally require fewer samples for training, $\mathrm{AD}$ allows an optimization algorithm to run for more iterations within the same time restrictions.

Properties of $\mathrm{AD}_\mathcal{D}$ highly depend on the family $\mathcal{D}$, and choice of the latter might either negatively or positively impact convergence of a particular optimization algorithm. Figure~\ref{fig:examples} demonstrates both cases: here, we evaluate $\mathrm{JSD}$ and four variants of $\mathrm{AD}_\mathcal{D}$ on two synthetic examples. In each example, the generator produces a rotated version of the ground-truth distribution and is parameterized by the angle of rotation (ground-truth distributions and examples of generator distributions are shown in Fig.~\ref{fig:roll-example}~and~Fig.~\ref{fig:xor-example}). In Fig.~\ref{fig:roll-gbdt}~and~Fig.~\ref{fig:roll-nn} $\mathrm{AD}$ shows behaviour similar to that of $\mathrm{JSD}$ (both being monotonous and maintaining a significant slope in the respective ranges). In Fig.~\ref{fig:xor-gbdt}, both variants of $\mathrm{AD}$ introduce an additional local minimum, which is expected to impact convergence of gradient-based algorithms negatively. In contrast, in Fig.~\ref{fig:xor-nn} neural-network-based $\mathrm{AD}$ with $l_2$ regularization stays monotonous in the range $[0, \pi / 2]$ and keeps a noticeable positive slope, in contrast to saturated $\mathrm{JSD}$. The positive slope is expected to improve convergence of gradient-based algorithms and, possibly, some variants of Bayesian Optimization. In contrast, neural-network-based $\mathrm{AD}$ with dropout regularization behaves in a manner similar to adaptive divergences in Fig.~\ref{fig:xor-gbdt}.

\begin{figure}
    \centering
    \begin{subfigure}[b]{0.32\textwidth}
        \includegraphics[width=\textwidth]{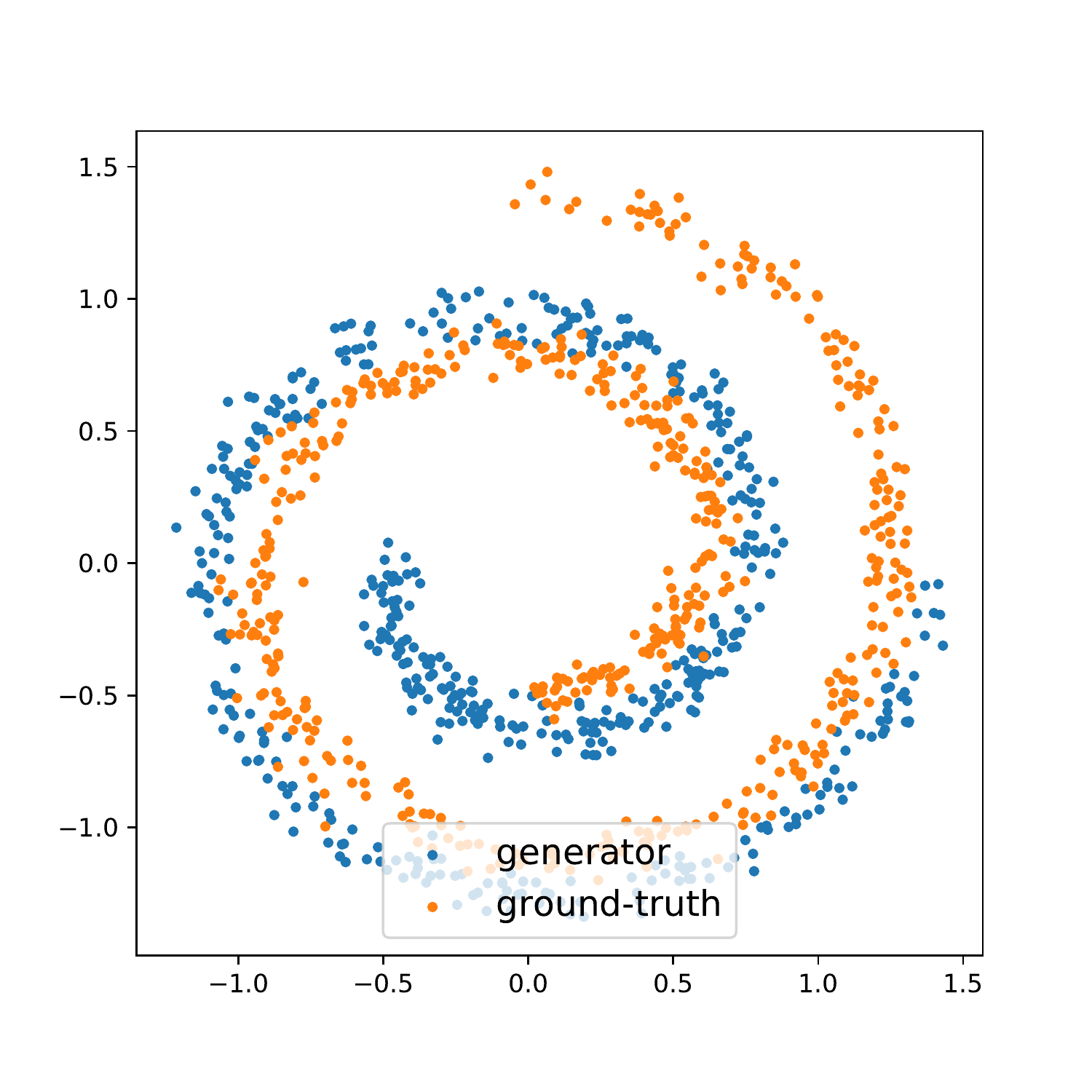}
        \caption{Example configuration.}
        \label{fig:roll-example}
    \end{subfigure}
    ~
    \begin{subfigure}[b]{0.32\textwidth}
        \includegraphics[width=\textwidth]{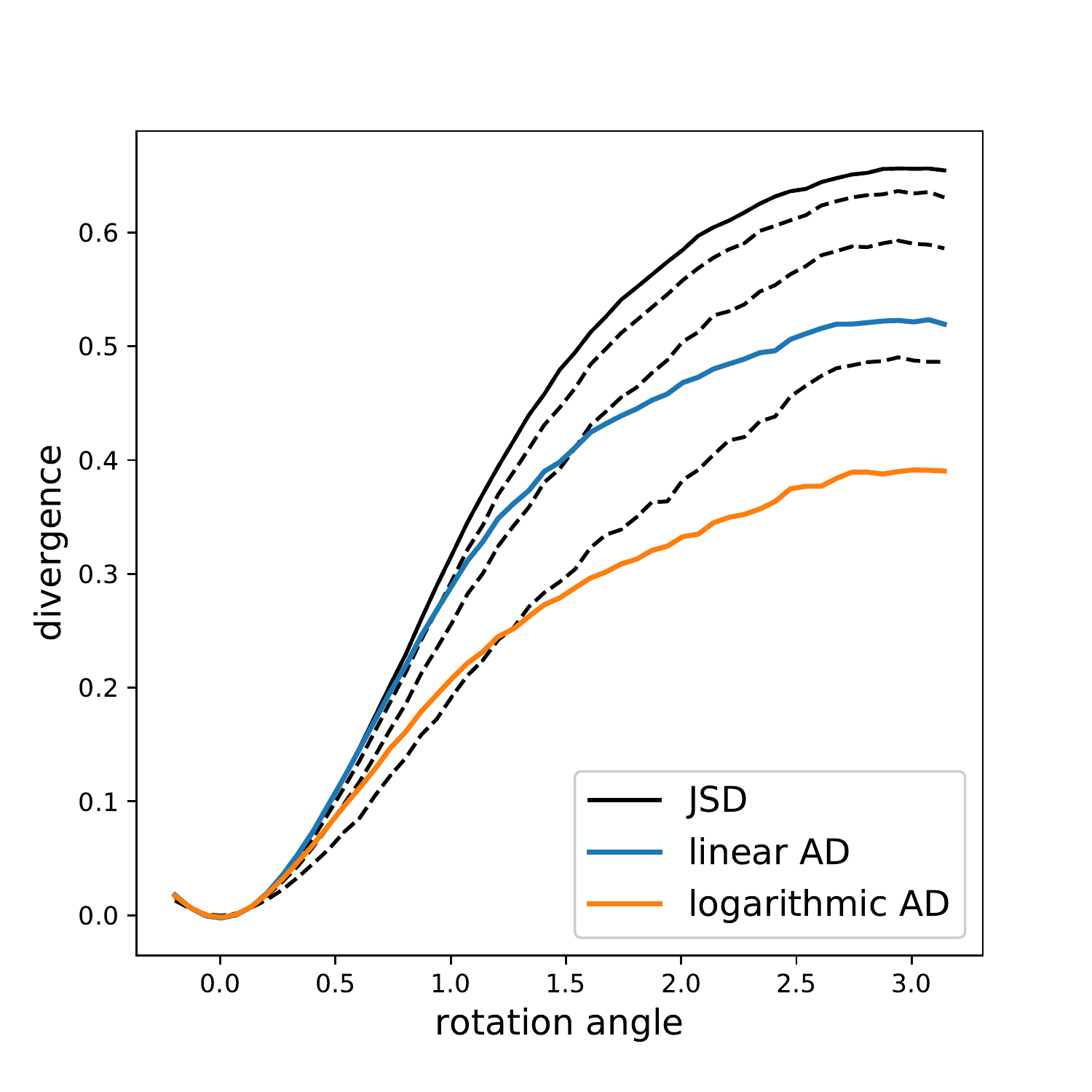}
        \caption{Gradient Boosting.}
        \label{fig:roll-gbdt}
    \end{subfigure}
    ~
    \begin{subfigure}[b]{0.32\textwidth}
        \includegraphics[width=\textwidth]{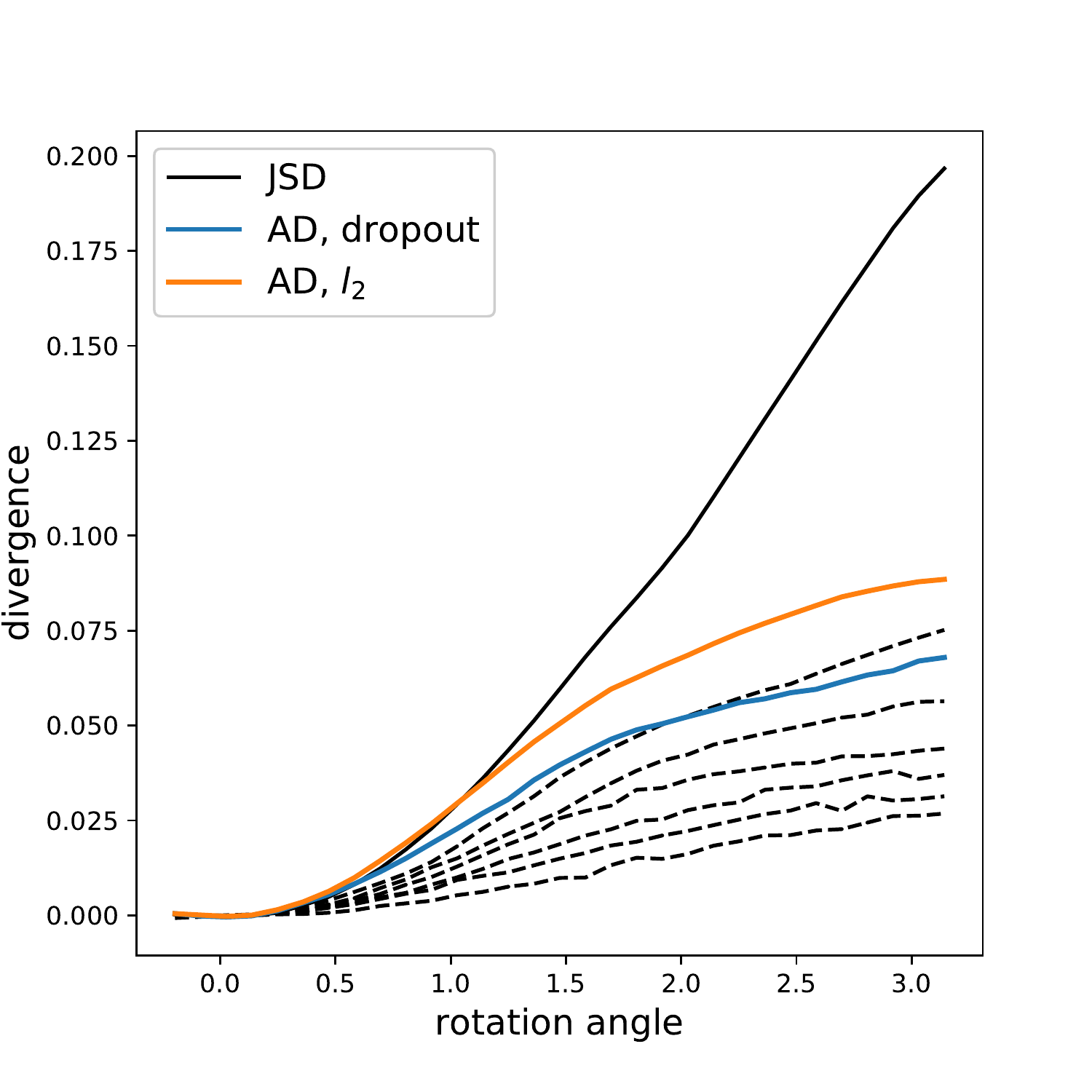}
        \caption{Neural Networks.}
        \label{fig:roll-nn}
    \end{subfigure}
    
    \begin{subfigure}[b]{0.32\textwidth}
        \includegraphics[width=\textwidth]{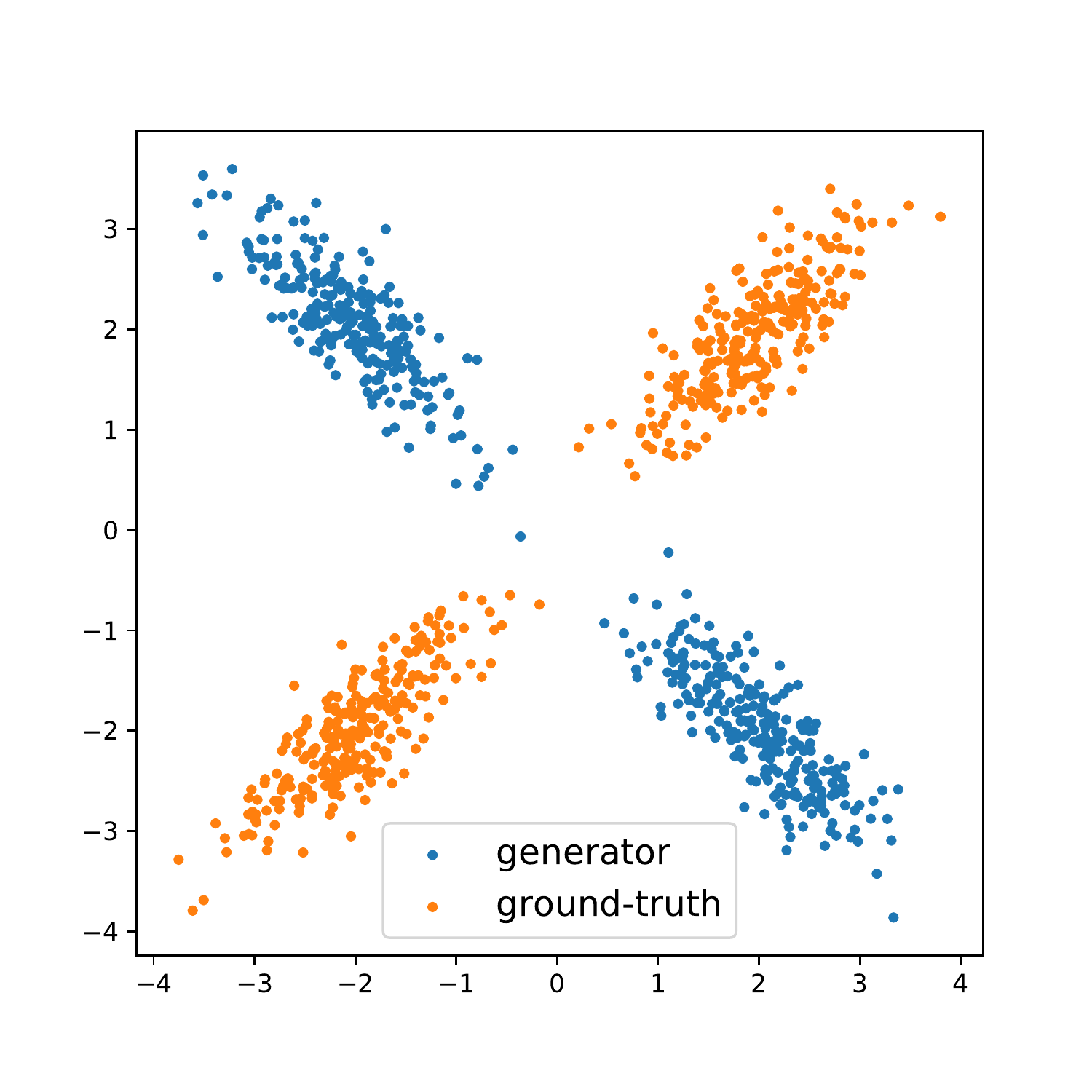}
        \caption{Example configuration.}
        \label{fig:xor-example}
    \end{subfigure}
    ~
    \begin{subfigure}[b]{0.32\textwidth}
        \includegraphics[width=\textwidth]{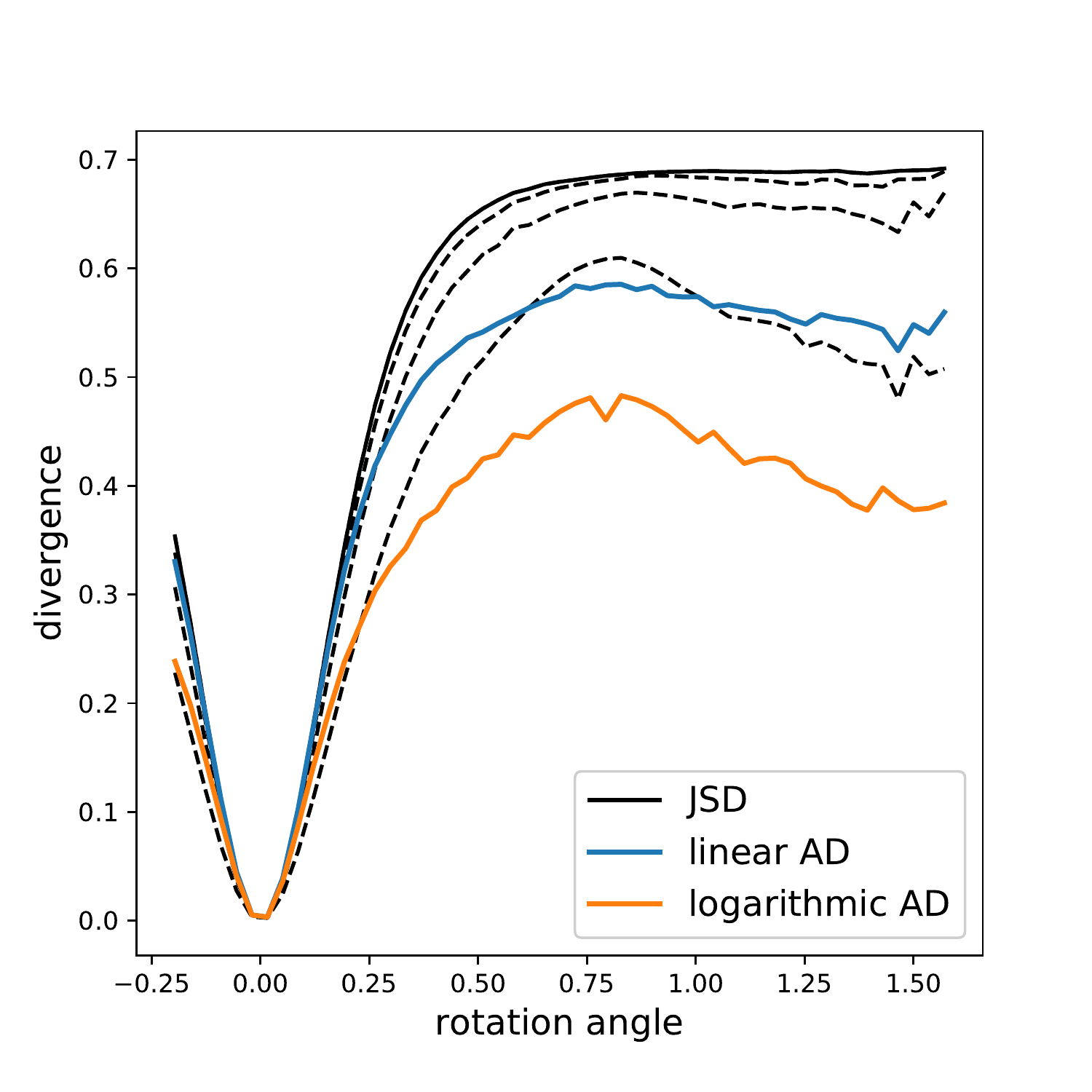}
        \caption{Gradient Boosting.}
        \label{fig:xor-gbdt}
    \end{subfigure}
    ~
    \begin{subfigure}[b]{0.32\textwidth}
        \includegraphics[width=\textwidth]{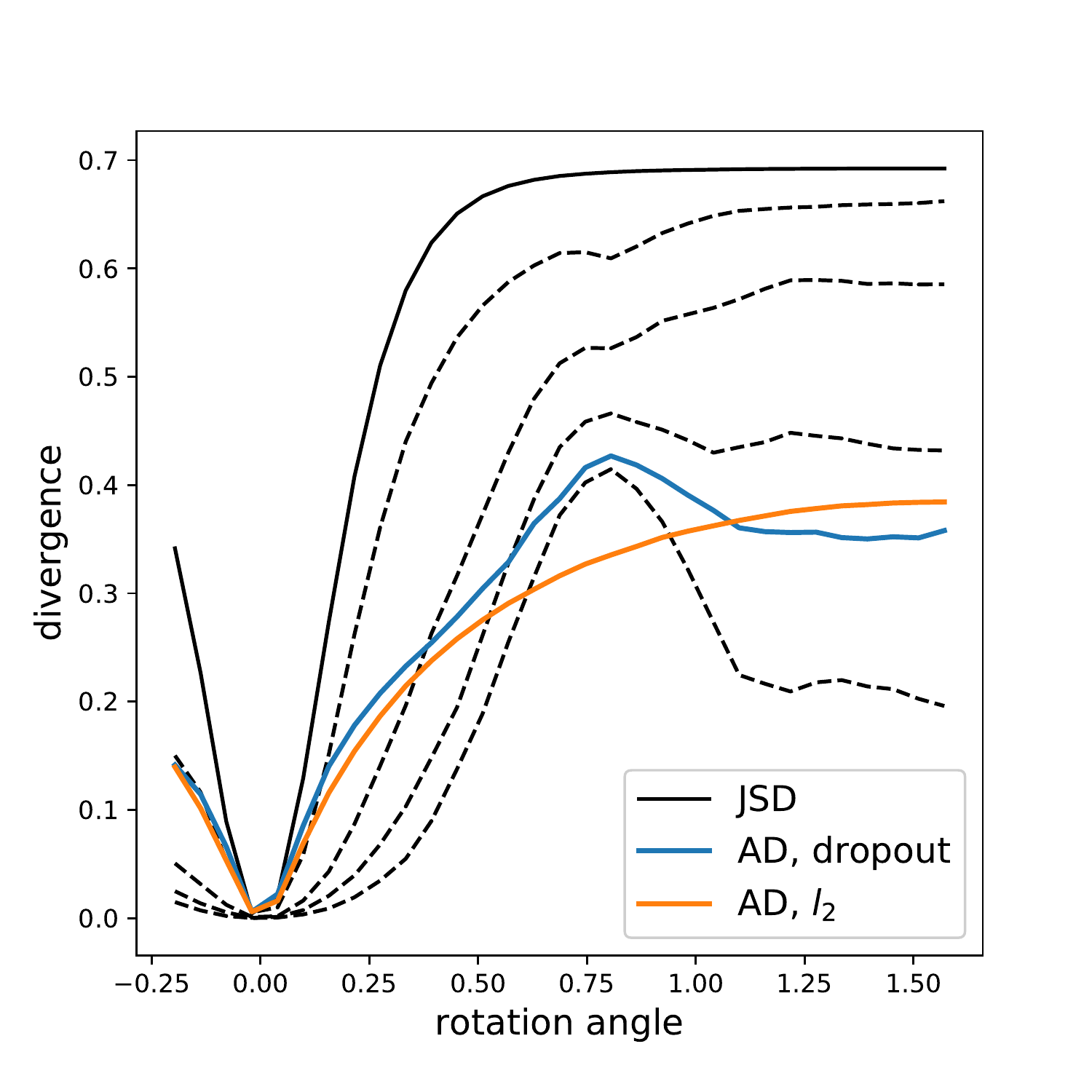}
        \caption{Neural Networks.}
        \label{fig:xor-nn}
    \end{subfigure}
    
    \caption{
        Synthetic examples.
        (a)~and~(d):~ground-truth distributions and example configurations of generators.
        Both generators are rotated versions of the corresponding ground-truth distributions.
        (b)~and~(e):~JSD --- Jensen-Shannon divergences estimated by Gradient Boosted Decision Trees with 500 trees of depth~3~(b), 100 trees of depth~3~(e); linear AD and logarithmic AD --- adaptive divergences based on  the same models as JSD with linear and logarithmic capacity functions, dashed lines represent some pseudo-divergences from the families producing adaptive divergences. (c)~and~(f):~JSD --- Jensen-Shannon divergences estimated by fully-connected Neural Networks with one hidden layer with 64~units~(c) and 32~units~(f); AD, dropout and AD, $l_2$ --- adaptive divergences based on the same architectures as the one for JSD, with dropout and $l_2$ regularizations; dashed line represent some of pseudo-divergences from the dropout-produced family.
        See section~\ref{sec:impl} for the implementation details.
    }
    \label{fig:examples}
\end{figure}


\section{Implementation}
\label{sec:impl}

A general algorithm for computing an adaptive divergence is presented in Algorithm~\ref{algo:ad}. This algorithm might be an expensive procedure as the algorithm probes multiple pseudo-divergences, and for each of these probes, generally, a model needs to be trained from scratch. However, two of the most commonly used machine learning models, boosting-based methods~\citep{friedman2001greedy} and Neural Networks, allow for more efficient estimation algorithms due to the iterative nature of training procedures for such models.

\subsection{Gradient Boosted Decision Trees}

Gradient Boosted Decision Trees~\citep{friedman2001greedy} (GBDT) and, generally, boosting-based methods, being ensemble methods, intrinsically produce an ordered and complete with respect to Jensen-Shannon divergence family of pseudo-divergences in the manner similar to equation~\eqref{eq:pd-boost}.
This allows for an efficient $\mathrm{AD}$ estimation procedure shown by algorithm~\ref{algo:ad-gbdt}. Here, the number of base estimators serves as capacity of pseudo-divergences, and mapping to $\alpha \in [0, 1]$ is defined through an increasing capacity function $c: \mathbb{Z}_+ \to [0, 1]$ \footnote{Technically, this function should be extended on $[0, +\infty)$ to be in agreement with definition~2.}.

\begin{algorithm}[tbp]
    \caption{Boosted adaptive divergence}
    \label{algo:ad-gbdt}
    \begin{algorithmic}
        \Require
            $X_P$, $X_Q$ --- samples from distributions $P$ and $Q$,
            $B$ --- base estimator training algorithm,
            $N$ --- maximal size of the ensemble,
            $c : \mathbb{Z}_+ \to [0, 1]$ --- capacity function;
            $\rho$ --- learning rate;

        \State $F_0 \leftarrow 1 / 2$
        \State $i \leftarrow 0$
        \State $L_0 \leftarrow \log 2$
        \For{$i = 1, \dots, N$}
            \If{$L_i > c(i) \log 2$}
                \State $F_{i + 1} \leftarrow F_i + \rho \cdot B(F_{i}, X_P, X_Q)$
                \State $L_{i + 1} \leftarrow L(F_{i + 1}, X_P, X_Q)$
                \State $i \leftarrow i + 1$
            \Else
                \State \Return{$\log 2 - L_i$}
            \EndIf
        \EndFor
        
        \State \Return{$\log 2 - L_N$}
    \end{algorithmic}
\end{algorithm}

In our experiments, for ensembles of maximal size $N$, we use the following capacity functions:
\begin{eqnarray}
    \text{linear capacity: } & c(i)& = c_0 \frac{i}{N}; \label{eq:gbdt-lincapacity}\\
    \text{logarithmic capacity: } & c(i)& = c_0 \frac{\log (i + 1)}{\log (N + 1)} \label{eq:gbdt-logcapacity}.
\end{eqnarray}

Notice, however, that Equation~\eqref{eq:pd-boost} defines a discrete variant of $\mathrm{AD}$, which most certainly will result in a discontinuous function\footnote{Note, that introducing a continuous approximation of the ensemble by, for example, varying learning rate for the last base estimator in the current ensemble from $0$ to $\rho$, eliminates discontinuity of $\mathrm{AD}$.}. This effect can be seen on Fig.~\ref{fig:xor-gbdt}.

\subsection{Neural Networks}

There is a number of ways to regulate the capacity of a neural network, in this work, we use well-established dropout regularization with rescaling of layers' outputs (similar to weight rescaling suggested by~\cite{srivastava2014dropout}). It is clear that setting dropout probability $p$ to 0 results in an unregularized network, while $p = 1$ effectively restricts classifier to a constant output and intermediate values of $p$ produce models in between these extreme cases. Additionally, we examine $l_2$ regularization on network weights. We use a linear capacity function for dropout regularization: $c(\alpha) = 1 - \alpha$, and a logarithmic one for $l_2$ regularization: $c(\alpha) = -\log(\alpha)$.

In our experiments, we observe that unregularized networks require significantly more samples to be properly trained than regularized ones. We suggest to use additional independent regularization, in this work, following~\cite{avo} we use gradient regularization $R_1$ suggested by~\cite{mescheder2018training}. Note, that a discriminator trained with such regularization no longer produces JSD estimations. Nevertheless, the resulting function is a proper divergence~\citep{mescheder2018training}, and all results in this work still hold with respect to such divergences.

To produce a family of pseudo-divergences, the proposed algorithm varies the strength of the regularization depending on the current values of the cross-entropy. The values of the loss function are estimated with exponential moving average over losses on mini-batches during iterations of Stochastic Gradient Descent, with the idea that, for slowly changing loss estimations and small enough learning rate, network training should converge~\citep{liu2018darts}. We find that initializing exponential moving average with $\log 2$, which corresponds to the absent regularization, works best. The proposed procedure is outlined in algorithms~\ref{algo:ad-nn-dropout}~and~\ref{algo:ad-nn-reg}.

\begin{algorithm}[ht]
    \caption{Adaptive divergence estimation by a dropout-regularized neural network}
    \label{algo:ad-nn-dropout}
    \begin{algorithmic}
        \Require $X_P$, $X_Q$ --- samples from distributions $P$ and $Q$;\\
            $f_\theta : \mathcal{X} \times \mathbb{R} \to \mathbb{R}$ --- neural network with parameters~$\theta \in \Theta$, the second argument represents dropout probability and is zero if unspecified;
            $c$~---~capacity function;\\
            $\rho$~---~exponential average coefficient;\\
            $\beta$~---~coefficient for $R_1$ regularization;\\
            $\gamma$~---~learning rate of SGD.\\

        \State $L_\mathrm{acc} \leftarrow \log 2$
    
        \While{not converged}
            \State $x_P \leftarrow \text{sample}(X_P)$
            \State $x_Q \leftarrow \text{sample}(X_Q)$
            \State $\zeta \leftarrow c\left(1 - \frac{L_\mathrm{acc}}{\log 2}\right)$
            \State $g_0 \leftarrow \nabla_\theta L(f_\theta(\cdot, \zeta), x_P, x_Q)$
            \State $g_1 \leftarrow \nabla_\theta \|\nabla_\theta f_\theta(x_P)\|^2$
            \State $L_\mathrm{acc} \leftarrow \rho \cdot L_\mathrm{acc} + (1 - \rho) \cdot L(f_\theta, x_P, x_Q)$
            \State $\theta \leftarrow \theta - \gamma \left(g_0 + \beta g_1 \right)$
        \EndWhile
        \State \Return $\log 2 - L(f_\theta, X_P, X_Q)$
    \end{algorithmic}
\end{algorithm}

\begin{algorithm}[ht]
    \caption{Adaptive divergence estimation by a regularized neural network}
    \label{algo:ad-nn-reg}
    \begin{algorithmic}
        \Require $X_P$, $X_Q$ --- samples from distributions $P$ and $Q$;\\
            $f_\theta : \mathcal{X} \to \mathbb{R}$ --- neural network with parameters~$\theta \in \Theta$;\\
            $R: \Theta \to \mathbb{R}$ --- regularization function;
            $c$~---~capacity function;\\
            $\rho$~---~exponential average coefficient;\\
            $\beta$~---~coefficient for $R_1$ regularization;\\
            $\gamma$~---~learning rate of SGD.\\

        \State $L_\mathrm{acc} \leftarrow \log 2$
    
        \While{not converged}
            \State $x_P \leftarrow \text{sample}(X_P)$
            \State $x_Q \leftarrow \text{sample}(X_Q)$
            \State $\zeta \leftarrow c\left(1 - \frac{L_\mathrm{acc}}{\log 2}\right)$
            \State $g_0 \leftarrow \nabla_\theta \left[ L(f_\theta, x_P, x_Q) + \zeta \cdot R(f_\theta) \right]$
            \State $g_1 \leftarrow \nabla_\theta \|\nabla_\theta f_\theta(x_P)\|^2$
            \State $L_\mathrm{acc} \leftarrow \rho \cdot L_\mathrm{acc} + (1 - \rho) \cdot L(f_\theta, x_P, x_Q)$
            \State $\theta \leftarrow \theta - \gamma \left(g_0 + \beta g_1 \right)$
        \EndWhile
        \State \Return $\log 2 - L(f_\theta, X_P, X_Q)$
    \end{algorithmic}
\end{algorithm}

\section{Experiments}

We evaluate the performance of adaptive divergences with two black-box optimization algorithms, namely Bayesian Optimization and Adversarial Variational Optimization\footnote{Code of the experiments is available at \url{https://github.com/HSE-LAMBDA/rapid-ao/}}. As computational resources spent by simulators are of our primary concern, we measure convergence of Adversarial Optimization with respect to the number of samples generated by the simulation.
Each task is presented by a parametrized generator, 'real-world' samples are drawn from the same generator with some nominal parameters. Optimization algorithms are expected to converge to these nominal parameters.

To measure the number of samples required to estimate a divergence, we search for the minimal number of samples such that the difference between train and validation losses is within $10^{-2}$ for Gradient Boosted Decision Trees and $5 \cdot 10^{-2}$ for Neural Networks\footnotemark. As a significant number of samples is involved in loss estimation, for simplicity, we ignore uncertainties associated with finite sample sizes. For GBDT, we utilize a bisection root-finding routine to reduce time spent on retraining classifiers; however, for more computationally expensive simulators, it is advised to gradually increase the size of the training set until the criterion is met.

\footnotetext{This procedure requires generating additional validation set of the size similar to that of the training set, which might be avoided by, e.g., using Bayesian inference, or cross-validation estimates.}

As the performance of Bayesian Optimization is influenced by choice of the initial points (in our experiments, 5 points uniformly drawn from the search space), each experiment involving Bayesian Optimization is repeated 100 times, and aggregated results are reported.

\subsection{XOR-like synthetic data}

This task repeats one of the synthetic examples presented in Fig.~\ref{fig:xor-example}: ground truth distribution is an equal mixture of two Gaussian distributions, the generator produces a rotated version of the ground-truth distribution with the angle of rotation being the single parameter of the generator. The main goal of this example is to demonstrate that, despite significant changes in the shape of the divergence, global optimization algorithms, like Bayesian Optimization, can still benefit from the fast estimation procedures offered by adaptive divergences.

For this task, we use an adaptive divergence based on Gradient Boosted Decision Trees (100 trees with the maximal depth of 3) with linear and logarithmic capacity functions given by Equations~\eqref{eq:gbdt-lincapacity}~and~\eqref{eq:gbdt-logcapacity} and $c_0~=~1/4$. Gaussian Process Bayesian Optimization with Matern kernel ($\nu = 3 / 2$ and scaling from $[10^{-3}, 10^{3}]$ automatically adjusted by Maximum Likelihood fit) is employed as optimizer.

\begin{figure}[h]
    \centering
    \includegraphics[width=\textwidth]{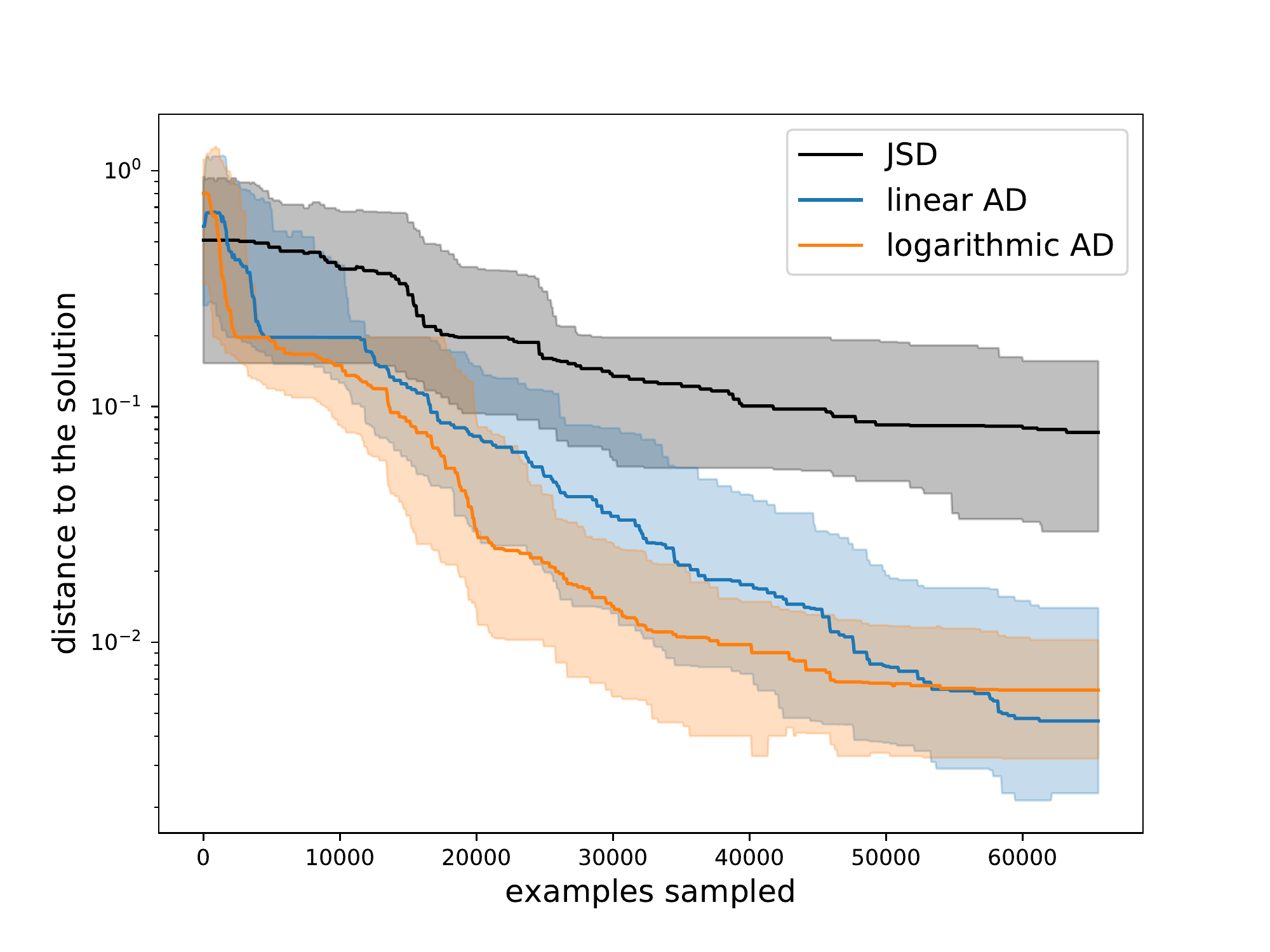}
    \caption{XOR-like synthetic example, Gradient Boosted Decision Trees. Convergence of Bayesian Optimization on: Jensen-Shannon divergence (marked as JSD), adaptive divergences with a linear capacity function (marked as linear AD), and a logarithmic capacity function (logarithmic AD). Each experiment was repeated 100 times; curves are interpolated, median curves are shown as solid lines, bands indicate first and third quartiles.}
    \label{fig:xor-ABO}
\end{figure}

Convergence of the considered divergences is shown in Fig.~\ref{fig:xor-ABO}. As can be seen from the results, given the same budget, both variants of adaptive divergence converge on parameters around an order of magnitude closer to the optimum than traditional JSD. Notice, that initial rapid progress slows as optimizer approaches the optimum, and the slope of the curves becomes similar to that of JSD: this can be explained by AD approaching JSD as probed distributions become less distinguishable from the ground-truth one.

\subsection{Pythia hyper-parameter tuning}
This task is introduced by~\cite{ilten2017event} and involves tuning hyper-parameters of the Pythia event generator, a high-energy particle collision simulation used at CERN. For this task, electron-positron collision are simulated at a center-of-mass energy 91.2 GeV, the nominal parameters of Pythia generator are set to the values of the Monash tune~\citep{monash}. Various physics-motivated statistics of events are used as observables, with a total of more than 400 features. The same statistics were originally used to obtain the Monash tune. For the purposes of this work, we consider one hyper-parameter, namely alphaSValue, with the nominal value of $0.1365$ and search range $[0.06, 0.25]$. 

We repeat settings of the experiment described by~\cite{ilten2017event}, with the only difference, that observed statistics are collected on the per-event basis instead of aggregating them over multiple events.  We employ Gradient Boosting over Oblivious Decision Trees (CatBoost implementation~\citep{prokhorenkova2018catboost}) with 100 trees of depth 3 and other parameters set to their default values. We use Gaussian Process Bayesian Optimization with Matern kernel ($\nu = 3 / 2$ and scaling from $[10^{-3}, 10^{3}]$ automatically adjusted by Maximum Likelihood fit) as optimizer. Comparison of unmodified Jensen-Shannon divergence with adaptive divergences with linear and logarithmic capacity functions (defined by Equations~\eqref{eq:gbdt-lincapacity}~and~\eqref{eq:gbdt-logcapacity} and $c_0~=~1/4$) presented on Fig.~\ref{fig:tunemc}.

\begin{figure}[h]
    \centering
    \includegraphics[width=0.85\textwidth]{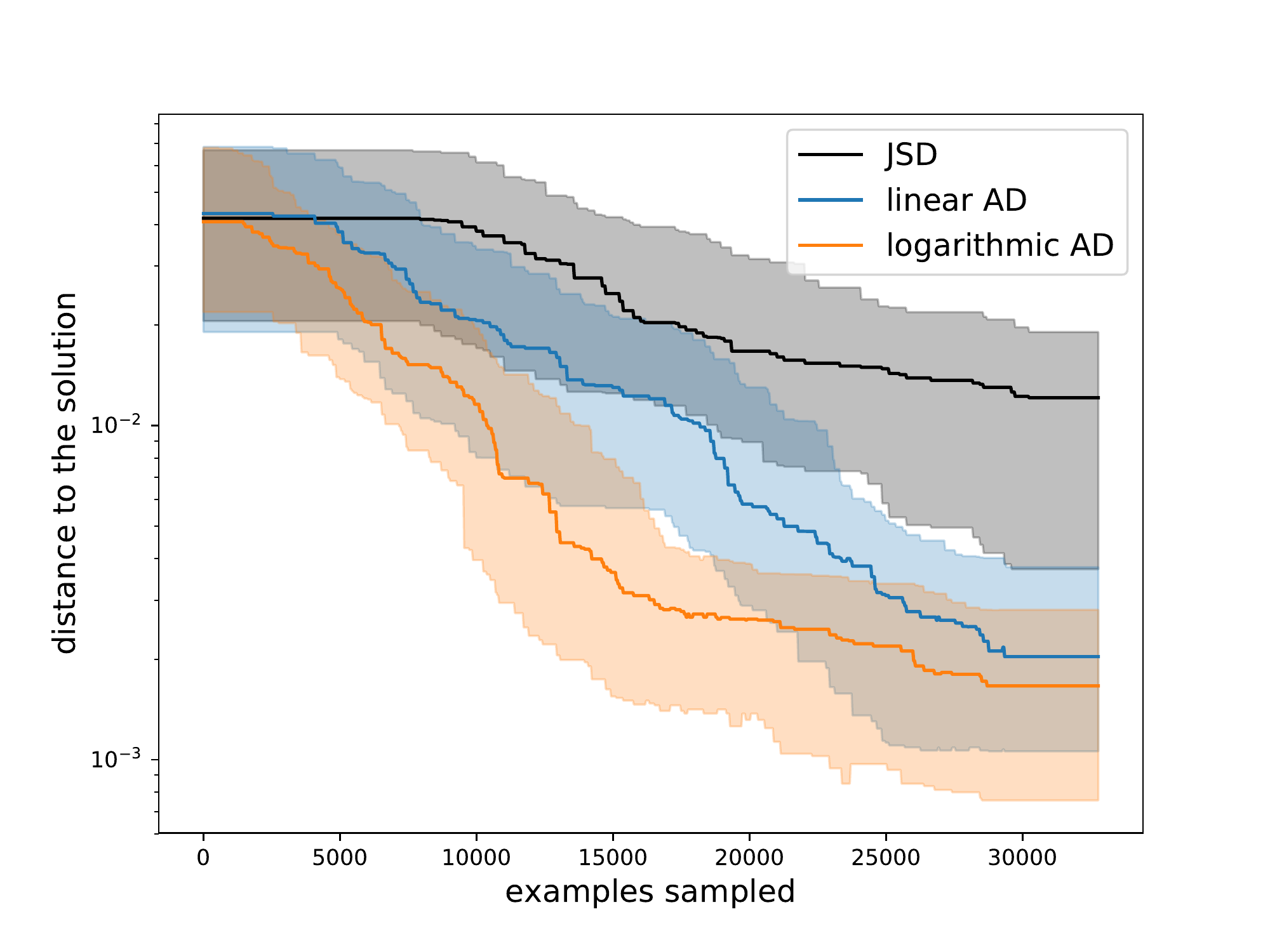}
    \caption{Pythia hyper-parameter tuning, CatBoost. Convergence of Bayesian Optimization on: Jensen-Shannon divergence (marked as JSD), adaptive divergences with a linear capacity function (marked as linear AD) and a logarithmic capacity function (logarithmic AD). Each experiment was repeated 100 times, curves are interpolated, median curves are shown as solid lines, bands indicate 25th and 75th percentiles.}
    \label{fig:tunemc}
\end{figure}

Results indicate that, given the same budget, Bayesian Optimization over adaptive divergences yields solutions about an order of magnitude closer to the nominal value than Jensen-Shannon divergence. Additionally, notice that the slope of the convergence curves for AD gradually approaches that of AD as the proposal distributions become closer to the ground-truth one.

\subsection{Pythia-alignment}
\label{subsec:alignment}
In order to test the performance of adaptive divergences with Adversarial Variational Optimization, we repeat the Pythia-alignment experiment suggested by~\cite{avo}. The settings of this experiment are similar to the previous one. In this experiment, however, we consider a detector represented by a $32\times32$ spherical grid with cells uniformly distributed in pseudorapidity $\nu \in [-5, 5]$ and azimuthal angle $\phi \in [-\pi, \pi]$ space. Each cell of the detector records the energy of particles passing through it. The detector has 3 parameters: $x, y, z$-offsets of the detector center relative to the collision point, where $z$-axis is placed along the beam axis, the nominal offsets are zero, and the initial guess is $(0.75, 0.75, 0.75)$. Fig.~\ref{fig:pythia-alignment-examples} shows averaged detector responses for the example configurations and samples from each of these configurations.

\begin{figure}[ht]
    \centering
    \includegraphics[width=\textwidth]{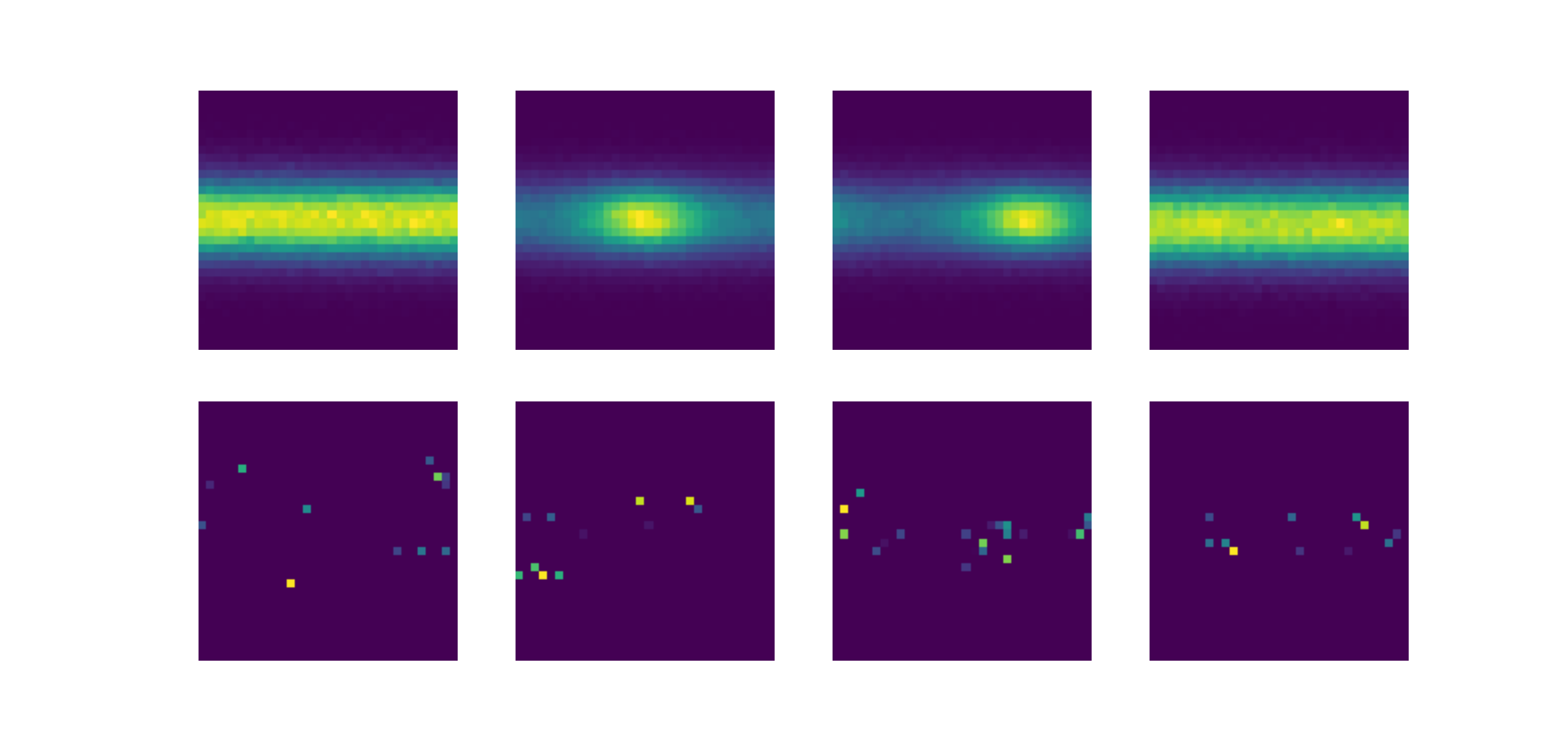}
    \caption{
        Illustration of the Pythia-alignment task. (Top row, from left to right) aggregated events for zero offset (the nominal configuration), 0.25 offset along $x$-axis, $y$-axis and $z$-axis. (Bottom row) single-event examples from the corresponding configurations above.
    }
    \label{fig:pythia-alignment-examples}
\end{figure}

For this task, a 1-hidden-layer Neural Network with 32 hidden units and ReLU activation function is employed. $R_1$ regularization, proposed by~\cite{mescheder2018training}, with the coefficient $10$, is used for the proposed divergences and the baseline. Adam optimization algorithm~\citep{kingma2014adam} with learning rate $10^{-2}$ is used to perform updates of the search distribution. We compare the performance of two variants of adaptive divergence (dropout and $l_2$ regularization) described in Section~\ref{sec:impl}.

\begin{figure}[ht]
    \centering
    \includegraphics[width=0.85\textwidth]{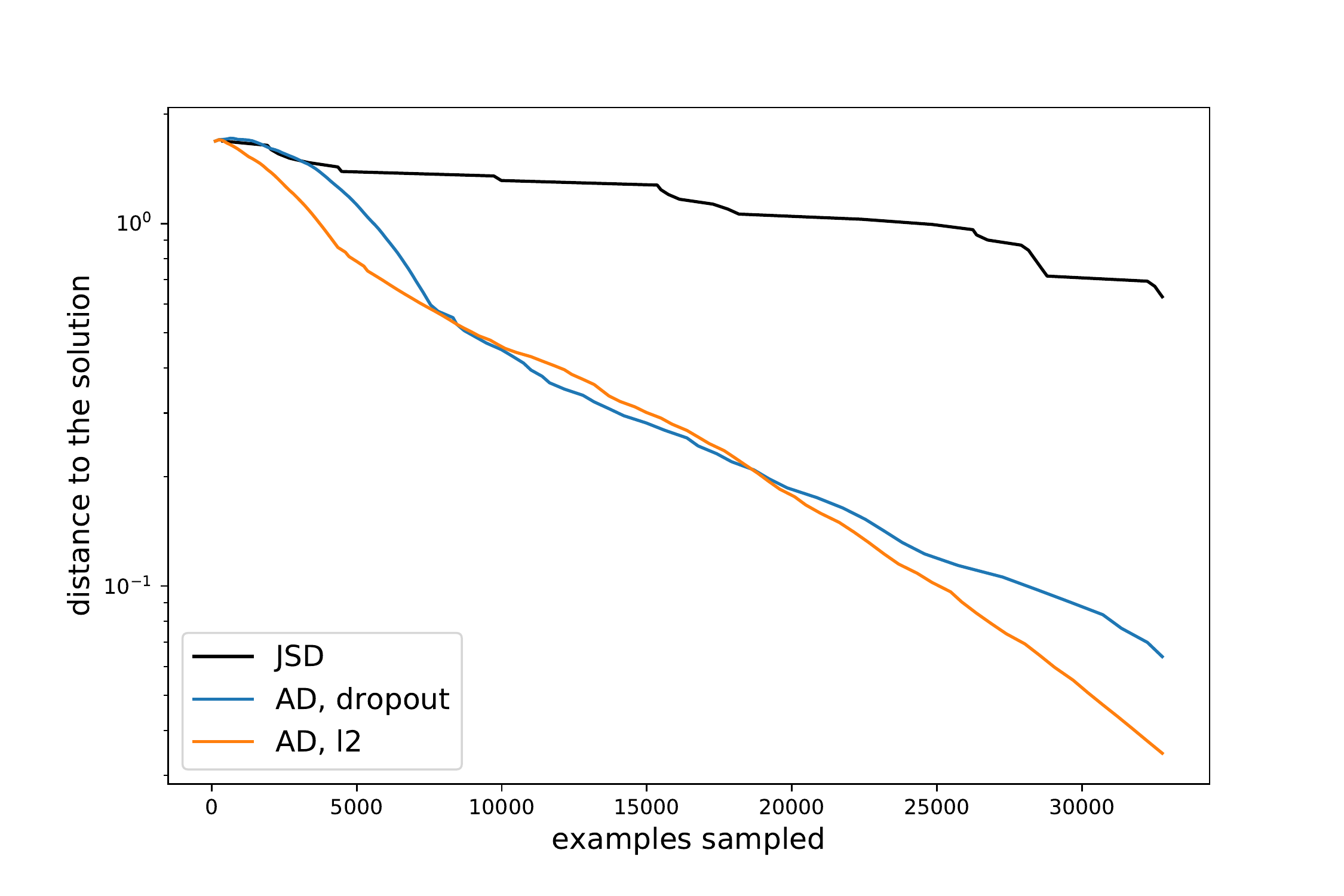}
    \caption{Pythia-alignment, Neural Networks. Convergence of Adversarial Variational Optimization on:
    adaptive divergence produced by $l_2$ regularization (AD, $l_2$), dropout regularization (AD, dropout), and the baseline divergence with constant $R_1$ regularization (marked as JSD).}
    \label{fig:pythia-alignment}
\end{figure}

Results are shown in Fig.~\ref{fig:pythia-alignment}. Given the same budget, both variants of adaptive divergence show a significant acceleration in contrast to the baseline divergence with only $R_1$ regularization. Note that the acceleration is even more pronounced in comparison to JSD estimated by an unregularized network: in our experiments, within the budget, we have not been able to consistently achieve the set level of agreement between train and test losses with the unregularized network.

\section{Discussion}
To the best knowledge of the authors, this work is the first one that explicitly addresses computational costs of Adversarial Optimization for expensive generators. Interestingly, several recent developments, like Progressive~GAN~\citep{karras2017progressive} and ChainGAN~\citep{hossain2018chaingan}, use multiple discriminators of increasing capacity; however, this is done mainly to compensate for the growing capacity of the generators and, probably, not for reducing computational cost.

Several recent papers propose improving stability of Adversarial Optimization by employing divergences other than Jensen-Shannon~\citep{gulrajani2017improved, arjovsky2017wasserstein, bellemare2017cramer}. Note that all results in this paper also hold for any divergence that can be formulated as an optimization problem, including Wasserstein~\citep{arjovsky2017wasserstein} and Cramer~\citep{bellemare2017cramer} distances. It can be demonstrated by adjusting Definition~2 and repeating the proof of Theorem~1 for a new divergence; presented algorithms also require only minor adjustments.

Multiple works introduce regularization~\citep{sonderby2016amortised, arjovsky2017principled, roth2017stabilizing, kodali2017convergence, mescheder2018training} for improving stability and convergence of Adversarial Optimization. Most of the standard regularization methods can be used to regulate model capacity in an adaptive divergence. Also, one can use these regularization methods in addition to adaptive divergence as any discriminator-based regularization effectively produces a new type of divergence. Pythia-alignment experiment (section~\ref{subsec:alignment}) demonstrates it clearly, where we use $R_1$ regularization with constant coefficient in addition to varying-strength dropout and $l_2$ regularization.

Finally, we would like to stress that the properties of adaptive divergences highly depend on the underlying families of pseudo-divergences, and interaction between AD and various proposed regularization schemes is subject to future research.

\section{Conclusion}

In this work, we introduce adaptive divergences, a family of divergences meant as an alternative to Jensen-Shannon divergence for Adversarial Optimization. Adaptive divergences generally require smaller sample sizes for estimation, which allows for a significant acceleration of Adversarial Optimization algorithms. These benefits were demonstrated on two fine-tuning problems involving Pythia event generator and two of the most popular black-box optimization algorithms: Bayesian Optimization and Variational Optimization. Experiments show that, given the same budget, adaptive divergences yield results up to an order of magnitude closer to the optimum than Jensen-Shannon divergence. Note, that while we consider physics-related simulations, adaptive divergences can be applied to any stochastic simulation.

Theoretical results presented in this work also hold for divergences other than Jensen-Shannon divergence.

\section*{Acknowledgments}
We wish to thank Mikhail Hushchyn, Denis Derkach and Marceline Ivanovna for useful discussions and suggestions on the text.

\section*{Funding Statement}
The research was carried out with the financial support of the Ministry of Science and Higher Education of Russian Federation within the framework of the Federal Target Program Research and Development in Priority Areas of the Development of the Scientific and Technological Complex of Russia for 2014-2020. Unique identifier RFMEFI58117X0023, agreement 14.581.21.0023 on 03.10.2017.

\clearpage
\bibliography{bibliography}

\begin{thebibliography}{}

\bibitem[Arjovsky and Bottou, 2017]{arjovsky2017principled}
Arjovsky, M. and Bottou, L. (2017).
\newblock Towards principled methods for training generative adversarial
  networks.

\bibitem[Arjovsky et~al., 2017]{arjovsky2017wasserstein}
Arjovsky, M., Chintala, S., and Bottou, L. (2017).
\newblock Wasserstein gan.
\newblock {\em arXiv preprint arXiv:1701.07875}.

\bibitem[Baydin et~al., 2018]{baydin2018efficient}
Baydin, A.~G., Heinrich, L., Bhimji, W., Gram-Hansen, B., Louppe, G., Shao, L.,
  Cranmer, K., Wood, F., et~al. (2018).
\newblock Efficient probabilistic inference in the quest for physics beyond the
  standard model.
\newblock {\em arXiv preprint arXiv:1807.07706}.

\bibitem[Bellemare et~al., 2017]{bellemare2017cramer}
Bellemare, M.~G., Danihelka, I., Dabney, W., Mohamed, S., Lakshminarayanan, B.,
  Hoyer, S., and Munos, R. (2017).
\newblock The cramer distance as a solution to biased wasserstein gradients.
\newblock {\em arXiv preprint arXiv:1705.10743}.

\bibitem[Bouhova-Thacker et~al., 2010]{bouhova2010atlas}
Bouhova-Thacker, E., Catmore, J., Cheatham, S., Chilingarov, A., Davidson, R.,
  de~Mora, L., Fox, H., Henderson, R., Hughes, G., Jones, R. W.~L., et~al.
  (2010).
\newblock The atlas simulation infrastructure.
\newblock {\em European Physical Journal C: Particles and Fields},
  70(3):823--874.

\bibitem[Choi et~al., 2018]{choi2018stargan}
Choi, Y., Choi, M., Kim, M., Ha, J.-W., Kim, S., and Choo, J. (2018).
\newblock Stargan: Unified generative adversarial networks for multi-domain
  image-to-image translation.
\newblock In {\em 2018 IEEE/CVF Conference on Computer Vision and Pattern
  Recognition}, pages 8789--8797. IEEE.

\bibitem[Dumoulin et~al., 2016]{dumoulin2016adversarially}
Dumoulin, V., Belghazi, I., Poole, B., Mastropietro, O., Lamb, A., Arjovsky,
  M., and Courville, A. (2016).
\newblock Adversarially learned inference.
\newblock {\em arXiv preprint arXiv:1606.00704}.

\bibitem[Friedman, 2001]{friedman2001greedy}
Friedman, J.~H. (2001).
\newblock Greedy function approximation: a gradient boosting machine.
\newblock {\em Annals of statistics}, pages 1189--1232.

\bibitem[Goodfellow et~al., 2014]{goodfellow2014generative}
Goodfellow, I., Pouget-Abadie, J., Mirza, M., Xu, B., Warde-Farley, D., Ozair,
  S., Courville, A., and Bengio, Y. (2014).
\newblock Generative adversarial nets.
\newblock In {\em Advances in neural information processing systems}, pages
  2672--2680.

\bibitem[Gulrajani et~al., 2017]{gulrajani2017improved}
Gulrajani, I., Ahmed, F., Arjovsky, M., Dumoulin, V., and Courville, A.~C.
  (2017).
\newblock Improved training of wasserstein gans.
\newblock In {\em Advances in neural information processing systems}, pages
  5767--5777.

\bibitem[Hossain et~al., 2018]{hossain2018chaingan}
Hossain, S., Jamali, K., Li, Y., and Rudzicz, F. (2018).
\newblock Chaingan: A sequential approach to gans.
\newblock {\em arXiv preprint arXiv:1811.08081}.

\bibitem[Ilten et~al., 2017]{ilten2017event}
Ilten, P., Williams, M., and Yang, Y. (2017).
\newblock Event generator tuning using bayesian optimization.
\newblock {\em Journal of Instrumentation}, 12(04):P04028.

\bibitem[Isola et~al., 2017]{isola2017image}
Isola, P., Zhu, J.-Y., Zhou, T., and Efros, A.~A. (2017).
\newblock Image-to-image translation with conditional adversarial networks.
\newblock In {\em Proceedings of the IEEE conference on computer vision and
  pattern recognition}, pages 1125--1134.

\bibitem[Karras et~al., 2017]{karras2017progressive}
Karras, T., Aila, T., Laine, S., and Lehtinen, J. (2017).
\newblock Progressive growing of gans for improved quality, stability, and
  variation.
\newblock {\em arXiv preprint arXiv:1710.10196}.

\bibitem[Kingma and Ba, 2014]{kingma2014adam}
Kingma, D.~P. and Ba, J. (2014).
\newblock Adam: A method for stochastic optimization.
\newblock {\em arXiv preprint arXiv:1412.6980}.

\bibitem[Kodali et~al., 2017]{kodali2017convergence}
Kodali, N., Abernethy, J., Hays, J., and Kira, Z. (2017).
\newblock On convergence and stability of gans.
\newblock {\em arXiv preprint arXiv:1705.07215}.

\bibitem[Li et~al., 2017]{li2017towards}
Li, J., Madry, A., Peebles, J., and Schmidt, L. (2017).
\newblock Towards understanding the dynamics of generative adversarial
  networks.
\newblock {\em arXiv preprint arXiv:1706.09884}.

\bibitem[Liu et~al., 2018]{liu2018darts}
Liu, H., Simonyan, K., and Yang, Y. (2018).
\newblock Darts: Differentiable architecture search.
\newblock {\em arXiv preprint arXiv:1806.09055}.

\bibitem[Louppe et~al., 2017]{avo}
Louppe, G., Hermans, J., and Cranmer, K. (2017).
\newblock Adversarial variational optimization of non-differentiable
  simulators.
\newblock {\em arXiv preprint arXiv:1707.07113}.

\bibitem[Mescheder et~al., 2018]{mescheder2018training}
Mescheder, L., Geiger, A., and Nowozin, S. (2018).
\newblock Which training methods for gans do actually converge?
\newblock In {\em International Conference on Machine Learning}, pages
  3478--3487.

\bibitem[Metz et~al., 2016]{metz2016unrolled}
Metz, L., Poole, B., Pfau, D., and Sohl-Dickstein, J. (2016).
\newblock Unrolled generative adversarial networks.
\newblock In {\em ICLR}.

\bibitem[Mockus, 2012]{mockus2012bayesian}
Mockus, J. (2012).
\newblock {\em Bayesian approach to global optimization: theory and
  applications}, volume~37.
\newblock Springer Science \& Business Media.

\bibitem[Prokhorenkova et~al., 2018]{prokhorenkova2018catboost}
Prokhorenkova, L., Gusev, G., Vorobev, A., Dorogush, A.~V., and Gulin, A.
  (2018).
\newblock Catboost: unbiased boosting with categorical features.
\newblock In {\em Advances in Neural Information Processing Systems}, pages
  6638--6648.

\bibitem[Radford et~al., 2015]{radford2015unsupervised}
Radford, A., Metz, L., and Chintala, S. (2015).
\newblock Unsupervised representation learning with deep convolutional
  generative adversarial networks.
\newblock {\em arXiv preprint arXiv:1511.06434}.

\bibitem[Roth et~al., 2017]{roth2017stabilizing}
Roth, K., Lucchi, A., Nowozin, S., and Hofmann, T. (2017).
\newblock Stabilizing training of generative adversarial networks through
  regularization.
\newblock In {\em Advances in neural information processing systems}, pages
  2018--2028.

\bibitem[Simonyan and Zisserman, 2014]{simonyan2014very}
Simonyan, K. and Zisserman, A. (2014).
\newblock Very deep convolutional networks for large-scale image recognition.
\newblock {\em arXiv preprint arXiv:1409.1556}.

\bibitem[Sj{\"o}strand et~al., 2015]{pythia82}
Sj{\"o}strand, T., Ask, S., Christiansen, J.~R., Corke, R., Desai, N., Ilten,
  P., Mrenna, S., Prestel, S., Rasmussen, C.~O., and Skands, P.~Z. (2015).
\newblock An introduction to pythia 8.2.
\newblock {\em Computer physics communications}, 191:159--177.

\bibitem[Sj{\"o}strand et~al., 2006]{pythia64}
Sj{\"o}strand, T., Mrenna, S., and Skands, P. (2006).
\newblock Pythia 6.4 physics and manual.
\newblock {\em Journal of High Energy Physics}, 2006(05):026.

\bibitem[Skands et~al., 2014]{monash}
Skands, P., Carrazza, S., and Rojo, J. (2014).
\newblock Tuning pythia 8.1: the monash 2013 tune.
\newblock {\em The European Physical Journal C}, 74(8):3024.

\bibitem[S{\o}nderby et~al., 2016]{sonderby2016amortised}
S{\o}nderby, C.~K., Caballero, J., Theis, L., Shi, W., and Husz{\'a}r, F.
  (2016).
\newblock Amortised map inference for image super-resolution.
\newblock {\em arXiv preprint arXiv:1610.04490}.

\bibitem[Srivastava et~al., 2014]{srivastava2014dropout}
Srivastava, N., Hinton, G., Krizhevsky, A., Sutskever, I., and Salakhutdinov,
  R. (2014).
\newblock Dropout: a simple way to prevent neural networks from overfitting.
\newblock {\em The journal of machine learning research}, 15(1):1929--1958.

\bibitem[Wierstra et~al., 2014]{wierstra2014natural}
Wierstra, D., Schaul, T., Glasmachers, T., Sun, Y., Peters, J., and
  Schmidhuber, J. (2014).
\newblock Natural evolution strategies.
\newblock {\em The Journal of Machine Learning Research}, 15(1):949--980.

\end{thebibliography}

\appendix

\section{Proof of Theorem~1}
\label{sec:proof}

\begin{theorem}
    \label{th:divergence}
    If $\mathrm{AD}_\mathcal{D}$ is an adaptive divergence produced by a complete and ordered with respect to Jensen-Shannon divergence family of pseudo-divergences $\mathcal{D}$, then for any two distributions $P$ and $Q$: $\mathrm{JSD}(P, Q) = 0$ if and only if  $\mathrm{AD}(P, Q) = 0$.
\end{theorem}

\begin{proof}
    For convenience, we repeat the definition of an adaptive divergence $\mathrm{AD}_\mathcal{D}$ here:
    \begin{equation}
        \mathrm{AD}_\mathcal{D}(P, Q) = \inf \left\{ D_\alpha(P, Q) \mid \mathrm{D}_\alpha(P, Q) \geq (1 - \alpha) \log 2\right\}.
    \end{equation}

    Firstly, we prove that from $\mathrm{JSD}(P, Q) = 0$ follows $\mathrm{AD}_\mathcal{D}(P, Q) = 0$.
    Due to Property~\ref{prop:family-complete}, $D_1(P, Q) = \mathrm{JSD}(P, Q) = 0$, therefore, $\forall \alpha \in [0, 1]: D_\alpha(P, Q) = 0$ due to Properties~\ref{prop:family-complete} (pseudo-divergences form a non-decreasing sequence) and~\ref{prop:pd-non-negativity} (non-negativity of pseudo-divergences), which, in turn, implies that $\mathrm{AD}(P, Q) = \inf \{ 0 \} = 0$.
    
    Secondly, we prove that from $\mathrm{AD}_\mathcal{D}(P, Q) = 0$ follows $\mathrm{JSD}(P, Q) = 0$. Let's assume that, for some $P$ and $Q$, $\mathrm{AD}(P, Q) = 0$, but $\mathrm{JSD}(P, Q) = C > 0$.
    Let us define the set of active capacities $\mathrm{A}_\mathcal{D}(P, Q)$ as follows:
    \begin{equation}
        \mathrm{A}_\mathcal{D}(P, Q) = \left\{ \alpha \mid \mathrm{D}_\alpha(P, Q) \geq (1 - \alpha) \log 2\right\}.
    \end{equation}
    Note, that for every proper family $\mathcal{D}$ and for every pair of $P$ and $Q$: $\{1\} \subseteq \mathrm{A}_\mathcal{D}(P, Q)$ and,
    if $\alpha \in \mathrm{A}_\mathcal{D}(P, Q)$ then $[\alpha, 1] \subseteq \mathrm{A}_\mathcal{D}(P, Q)$. The latter follows from Property~\ref{prop:family-ordered} (pseudo-divergences form a non-decreasing sequence) and the fact, that $(1 - \alpha)\log 2$ is a strictly decreasing function.
    
    The previous statement implies that there are three possible forms of $\mathrm{A}_\mathcal{D}(P, Q)$:
    \begin{enumerate}
        \item a single point: $\mathrm{A}_\mathcal{D}(P, Q) = \{1\}$;
        \item an interval: $\mathrm{A}_\mathcal{D}(P, Q) = [\beta, 1]$;
        \item a half-open interval: $\mathrm{A}_\mathcal{D}(P, Q) = (\beta, 1]$;
    \end{enumerate}
    for some $\beta \in [0, 1)$. The first case would contradict our assumptions, since $\mathrm{AD}_\mathcal{D}(P, Q) = \inf \{ D_1(P, Q) \} = C > 0$.
    To address the last two cases, note, that $\forall \alpha \in \mathrm{A}_\mathcal{D}(P, Q): D_\alpha(P, Q) \geq (1 - \beta) \log 2 > 0$ due to the definition of $\mathrm{A}_\mathcal{D}(P, Q)$. However, this implies that $\mathrm{AD}_\mathcal{D}(P, Q) = \inf \{ D_\alpha(P, Q) \mid \alpha \in \mathrm{A}_\mathcal{D}(P, Q) \} \geq (1 - \beta) \log 2 > 0$, which contradicts our assumptions.
    
    From the statements above, we can conclude that if $\mathrm{AD}_\mathcal{D}(P, Q) = 0$, then $\mathrm{JSD}(P, Q) = 0$.
    Combined with the previouly proven $\left(\mathrm{JSD}(P, Q) = 0 \right)\Rightarrow \left(\mathrm{AD}_\mathcal{D}(P, Q) = 0\right)$, this finishes the proof.
\end{proof}

\section{Formal definitions and proofs}
\label{sec:app-def}

\begin{definition}
    A model family $\mathcal{M} = \{ M_\alpha \subseteq \mathcal{F} \mid \alpha \in [0, 1] \}$ is complete and nested, if:
    \begin{description}
        \item[\namedlabel{prop:nested-zero}{(N0)}] $(x \mapsto 1 / 2) \in M_0$;
        \item[\namedlabel{prop:nested-complete}{(N1)}] $M_1 = \mathcal{F}$;
        \item[\namedlabel{prop:nested-ordered}{(N2)}] $\forall \alpha, \beta \in [0, 1]: (\alpha < \beta) \Rightarrow (M_{\alpha} \subset M_{\beta})$.
    \end{description}
\end{definition}

\begin{theorem}
    If a model family $\mathcal{M} = \{ M_\alpha \subseteq \mathcal{F} \mid \alpha \in [0, 1] \}$ is complete and nested, then the family $\mathcal{D} = \{ D_\alpha : \Pi(\mathcal{X}) \times \Pi(\mathcal{X}) \to \mathbb{R} \mid \alpha \in [0, 1] \}$, where:
    \begin{equation}
        D_\alpha(P, Q) = \log 2 - \inf_{f \in M_\alpha} L(f, P, Q), \label{eq:ad-set}
    \end{equation}
    is a complete and ordered with respect to Jensen-Shannon divergence family of pseudo-divergences.
\end{theorem}

\begin{proof}
Let's introduce function $f_0(x) = 1 / 2$. Now we prove the theorem by proving that the family satisfies all properties from Definition~\ref{def:family}.

\paragraph{Property~\ref{prop:family-pd}} Due to Properties~\ref{prop:nested-zero}~and~\ref{prop:nested-ordered}, $f_0$ is a member of each set $M_\alpha$. This implies, that $D_\alpha(P, Q) \geq 0$ for all $\alpha \in [0, 1]$. For $P = Q$, cross-entropy loss function $L(f, P, Q)$ achieves its minimum in $f = f_0$, therefore, $D_\alpha(P, Q) = 0$ if $P = Q$ for all $\alpha \in [0, 1]$. Therefore, for each $\alpha \in [0, 1]$ $D_\alpha$ is a pseudo-divergence.

\paragraph{Property~\ref{prop:family-ordered}}
From Property~\ref{prop:nested-ordered} follows, that for all $0 \leq \alpha < \beta \leq 1$:
\begin{equation*}
    D_{\alpha}(P, Q) = \log 2 - \inf_{f \in M_\alpha} L(f, P, Q) \geq \log 2 - \inf_{f \in M_\beta} L(f, P, Q) = D_\beta(P, Q).
\end{equation*}

\paragraph{Property~\ref{prop:family-complete}}
This property is directly follows from Property~\ref{prop:nested-complete} and Equation~\eqref{eq:ad-set}.
\end{proof}

\begin{definition}
    If $M$ is a parameterized model family $M = \{f(\theta, \cdot) : \mathcal{X} \to [0, 1] \mid \theta \in \Theta \}$, then a function $R : \Theta \to \mathbb{R}$ is a proper regularizer for the family $M$ if:
    \begin{description}
        \item[\namedlabel{prop:reg-non-negativity}{(R1)}] $\forall \theta \in \Theta: R(\theta) \geq 0$;
        \item[\namedlabel{prop:reg-zero}{(R2)}] $\exists \theta_0 \in \Theta: \left(f(\theta, \cdot) \equiv \frac{1}{2} \right) \land \left( R(\theta) = 0 \right)$.
    \end{description}
\end{definition}

\begin{theorem}
    If $M$ is a parameterized model family: $M = \{ f(\theta, \cdot) \mid \theta \in \Theta \}$ and $M = \mathcal{F}$, $R : \Theta \to \mathbb{R}$ is a proper regularizer for $M$ , and $c: [0, 1] \to [0, +\infty)$ is a strictly increasing function such, that $c(0) = 0$, then the family $\mathcal{D} = \{ D_\alpha : \Pi(\mathcal{X}) \times \Pi(\mathcal{X}) \to \mathbb{R} \mid \alpha \in [0, 1] \}$:
    \begin{eqnarray*}
        D_\alpha(P, Q) &=& \log 2 - \min_{\theta \in \Theta_\alpha(P, Q)} L(f(\theta, \cdot), P, Q);\\
        \Theta_\alpha(P, Q) &=&  \Argmin_{\theta \in \Theta} L^R_\alpha(\theta, P, Q);\\
        L^R_\alpha(\theta, P, Q) &=& L(f(\theta, \cdot), P, Q) + c(1 - \alpha) R(\theta);
    \end{eqnarray*}
    is a complete and ordered with respect to Jensen-Shannon divergence family of pseudo-divergences.
\end{theorem}

\begin{proof}
We prove the theorem by showing that the family $\mathcal{D}$ satisfies all properties from Definition~\ref{def:family}.

\paragraph{Property~\ref{prop:family-pd}}
Due to Property~\ref{prop:reg-zero}, there exists such $\theta_0$, that $f(\theta_0, \cdot) \equiv 1 / 2$ and $R(\theta_0) = 0$. Notice, that, for all $P$ and $Q$, $L^R_\alpha(\theta_0, P, Q) = \log 2$ and $L^R_\alpha(\theta, P, Q) \geq L(f(\theta, \cdot), P, Q)$, therefore, $D_\alpha(P, Q) \geq 0$ for all $P, Q \in \Pi(\mathcal{X})$ and for all $\alpha \in [0, 1]$. For the case $P = Q$, $\theta_0$ also delivers minimum to $L(f(\theta_0, \cdot), P, Q) + c(1 - \alpha) R(\theta_0)$, thus, $D_\alpha(P, Q) = 0$ if $P = Q$. This proves $D_\alpha$ to be a pseudo-divergence for all $\alpha \in [0, 1]$.

\paragraph{Property~\ref{prop:family-ordered}}
Let's assume that $0 \leq \alpha < \beta \leq 1$, yet, for some $P$ and $Q$, $D_\alpha(P, Q) > D_\beta(P, Q)$. The latter implies, that:
\begin{equation}
    \min_{\theta \in \Xi_\alpha} L(f(\theta, \cdot), P, Q) < \min_{\theta \in \Xi_\beta} L(f(\theta, \cdot), P, Q)\label{eq:reg-assumption};
\end{equation}
where: $\Xi_\alpha = \Theta_\alpha(P, Q)$ and $\Xi_\beta = \Theta_\beta(P, Q)$.
Let us pick some model parameters:
\begin{eqnarray*}
    \theta_\alpha &\in& \Argmin_{\theta \in \Xi_\alpha} L(f(\theta, \cdot), P, Q);\\
    \theta_\beta &\in& \Argmin_{\theta \in \Xi_\beta} L(f(\theta, \cdot), P, Q).
\end{eqnarray*}

Since $\theta_\beta \in \Xi_\beta$, then, by the definition of $\Theta_\beta(P, Q)$:
\begin{equation}
    L^R_\beta(\theta_\beta, P, Q) \leq L^R_\beta(\theta_\alpha, P, Q)\label{eq:reg-1}.
\end{equation}
From the latter and assumption~\eqref{eq:reg-assumption} follows, that $R(\theta_\beta) < R(\theta_\alpha)$. By the conditions of the theorem, $C = c(1 - \alpha) - c(1 - \beta) > 0$ and:
\begin{equation}
    C \cdot R(\theta_\beta) < C \cdot R(\theta_\alpha) \label{eq:reg-2}.
\end{equation}
Adding inequality~\eqref{eq:reg-1} to inequality~\eqref{eq:reg-2}:
\begin{equation*}
    L^R_\alpha(\theta_\beta, P, Q) < L^R_\alpha(\theta_\alpha, P, Q),
\end{equation*}
which contradicts the definition of $\theta_\alpha$. This, in turn, implies that the assumption~\eqref{eq:reg-assumption} contradicts conditions of the theorem.

\paragraph{Property~\ref{prop:family-complete}}
Since $c(0) = 0$ and $M = \mathcal{F}$, $D_1 = \mathrm{JSD}$ by the definition.
\end{proof}
\end{document}